\documentclass{article}
\usepackage{ijcai19}

\usepackage{times}
\usepackage{soul}
\usepackage{url}
\usepackage[hidelinks]{hyperref}
\usepackage[utf8]{inputenc}
\usepackage[small]{caption}
\usepackage{graphicx}
\usepackage{amsmath}
\usepackage{booktabs}
\usepackage{algorithm}
\usepackage{algorithmic}
\usepackage{graphicx}     
\usepackage{amsmath}
\usepackage{txfonts}
\usepackage{todonotes}

\newtheorem{THEOREM}{Theorem}
\newenvironment{theorem}{\begin{THEOREM} \hspace{-.85em} {\bf :} }%
                        {\end{THEOREM}}

\newtheorem{LEMMA}[THEOREM]{Lemma}
\newenvironment{lemma}{\begin{LEMMA} \hspace{-.85em} {\bf :} }%
                      {\end{LEMMA}}
\newtheorem{COROLLARY}[THEOREM]{Corollary}
                          {\end{COROLLARY}}
\newtheorem{PROPOSITION}[THEOREM]{Proposition}
\newenvironment{proposition}{\begin{PROPOSITION} \hspace{-.85em} {\bf :} }%
                            {\end{PROPOSITION}}
\newtheorem{DEFINITION}[THEOREM]{Definition}
\newenvironment{definition}{\begin{DEFINITION} \hspace{-.85em} {\bf :} \rm}%
                            {\end{DEFINITION}}
\newtheorem{CLAIM}[THEOREM]{Claim}
                            {\end{CLAIM}}
\newtheorem{EXAMPLE}[THEOREM]{Example}
\newenvironment{example}{\begin{EXAMPLE} \hspace{-.85em} {\bf :} \rm}%
                            {\end{EXAMPLE}}
\newtheorem{REMARK}[THEOREM]{Remark}
                            {\end{REMARK}}
							\newtheorem{NOTATION}[THEOREM]{Notation}
							                            {\end{NOTATION}}
\newenvironment{proof}{\noindent {\bf Proof:} \hspace{.677em}}%
                     {}

\newcommand{\bbox}{\vrule height7pt width4pt depth1pt}
\newcommand{\qed}{\bbox\vspace{0.1in}}
\DeclareMathAlphabet{\mathitbf}{OML}{cmm}{b}{it}

\newcommand{\sub}{_}
\def\su{^}

\newcommand{\nat}{{\mathbb{N}}}

\newcommand{\gt}{>}

\newcommand{\I}{{\cal I}}

\renewcommand{\L}{{\cal L}}

\newcommand{\M}{{\cal M}}
\newcommand{\N}{{\cal N}}

\newcommand{\U}{\mathbb{N}}
\newcommand{\UU}{{\cal U}}
\newcommand{\V}{{\cal V}}

\newcommand{\abs}[1]{ \llbracket  #1\rrbracket}
\newcommand{\set}[1]{\left\{ #1 \right\}}

\newcommand{\atom}{\textsc{ATOMS}}
\newcommand{\langs}{\L}
\newcommand{\atoms}{{\atom}}
\newcommand{\prop}{\textrm{proper}^+}

\newcommand{\blemma}{\begin{lemma}}

\newcommand{\elemma}{\end{lemma}}

\newcommand{\bthm}{\begin{theorem}}
\newcommand{\ethm}{\end{theorem}}
\newcommand{\bprf}{\begin{proof}}
\newcommand{\eprf}{\end{proof}}
\newcommand{\bpro}{\begin{proposition}}
\newcommand{\epro}{\end{proposition}}
\newcommand{\bi}{\begin{itemize}}
\newcommand{\ei}{\end{itemize}}
\newcommand{\be}{\begin{enumerate}}
\newcommand{\ee}{\end{enumerate}}
\newcommand{\beq}{\begin{equation}}
\newcommand{\eeq}{\end{equation}}
\newcommand{\bcase}{\begin{cases}}
\newcommand{\ecase}{\end{cases}}
\newcommand{\eps}{\epsilon}

\newcommand{\proper}{\textrm{proper}^+}

\begin{document}

	
\title{Implicitly Learning to Reason in First-Order Logic}

\author{%
 Vaishak Belle$^1$,
  Brendan Juba$^2$ \\[2ex] 
  $^1$University of Edinburgh, UK \& Alan Turing Institute, UK \\ 
 \texttt{vaishak@ed.ac.uk}\\
  $^2$Department of Computer Science \& Engineering,
  Washington University in St.\ Louis, USA \\  
  \texttt{bjuba@wustl.edu} 
}

\maketitle

\begin{abstract}
 We consider the problem of answering queries about formulas of first-order logic based on background knowledge partially represented explicitly as other formulas, and partially represented as examples independently drawn from a fixed probability distribution. PAC semantics, introduced by Valiant, is one rigorous, general proposal for learning to reason in formal languages: although weaker than classical entailment, it allows for a powerful model theoretic framework for answering queries while requiring minimal assumptions about the form of the distribution in question. To date, however, the most significant limitation of that approach, and more generally most machine learning approaches with robustness guarantees, is that the logical language is ultimately essentially propositional, with finitely many atoms. Indeed, the theoretical findings on the learning of relational theories in such generality have been resoundingly negative. This is despite the fact that first-order logic is widely argued to be most appropriate for representing human knowledge. 
In this work, we present a new theoretical approach to robustly learning to reason in first-order logic, and consider universally quantified clauses over a countably infinite domain. Our results exploit symmetries exhibited by constants in the language, and generalize the notion of implicit learnability to show how queries can be computed against (implicitly) learned first-order background knowledge. 

\end{abstract}

\renewcommand{\mathit}{\emph}

\section{Introduction} 
\label{sec:introduction}

The tension between \mathit{deduction} and \mathit{induction} is perhaps the most fundamental issue in areas such as philosophy, cognition and artificial intelligence. The deduction camp concerns itself with questions about the expressiveness of formal languages for capturing knowledge about the world, together with proof systems for reasoning from such {\it knowledge bases.} The learning camp attempts to generalize from examples about partial descriptions about the world. In an influential paper,  \cite{valiant2000robust} recognized that the challenge of learning should be integrated with deduction. In particular, he proposed a semantics to capture the quality possessed by the output of (probably approximately correct) PAC-learning algorithms when formulated in a logic. Although weaker than classical entailment, it allows for a powerful model theoretic framework for answering queries.




From the standpoint of
learning an  expressive logical knowledge base and reasoning with it, most PAC results are somewhat discouraging. For example, in agnostic learning \cite{kearns1994toward} where one does not require examples (drawn from an arbitrary distribution) to be fully consistent with learned sentences, efficient algorithms for learning conjunctions would yield
an efficient algorithm for PAC-learning DNF (also over arbitrary distributions), which current evidence suggests to be intractable \cite{daniely2016complexity}. Thus, it is not  surprising that when it comes to first-order logic (FOL), very little work tackles the problem in a general manner. This is despite the fact that FOL is widely argued to be most appropriate for representing human knowledge (e.g., \cite{some-philosophical-problems-from,the-role-of-logic-in-knowledge-representation,the-logic-of-knowledge-bases}). For example, \cite{cohen1994learnability} consider the problem of the learnability of description logics with equality constraints. While description logics are already restricted fragments of FOL in only allowing unary and some binary predicates, it is shown that such a fragment cannot be tractably learned, leading to the identification of syntactic restrictions for learning from positive examples alone. Analogously, when it comes to the learning of logic programs \cite{cohen1995polynomial}, which in principle may admit infinitely many terms, syntactic restrictions are also typical  \cite{first-order-jk-clausal-theories-are-pac-learnable}. 

In this work, we present new results on learning to reason in FOL knowledge bases. In particular, we consider the problem of answering queries about FOL  formulas  based on background knowledge partially represented explicitly as other formulas, and partially represented as examples independently drawn from a fixed probability distribution. Our results are based on a surprising observation made in \cite{juba2013implicit} about the advantages of eschewing the explicit construction of a hypothesis, leading to a paradigm of \mathit{implicit learnability.} Not only does it enable a form of agnostic learning while circumventing known barriers, it also avoids the design of an often restrictive and artificial choice for  representing hypotheses. (See, for example, \cite{khardon1999learning}, which is  similar in spirit in allowing declarative background knowledge but only permits constant-width clauses.)  In particular, implicit learning allows such learning from partially observed examples, which is commonplace when knowledge bases and/or queries address entities and relations not observed in the data used for learning. 

That work was limited to the propositional setting, however. Here, we develop a first-order logical generalization. Since reasoning in full FOL is undecidable we need to consider a fragment, but the fragment we identify and are able to learn and reason with is expressive and powerful. Consider that standard databases correspond to a maximally consistent and finite set of literals: every relevant atom is known to be true and stored in the database, or known to be false, inferred by (say) negation as failure. Our fragment corresponds to a consistent but infinite set of ground clauses, not necessarily maximal. 
To achieve the generalization, we  revisit the PAC semantics and exploit symmetries exhibited by constants in the language. Moreover, the underlying language is general in the sense that no restrictions are posed on clause length, predicate arity, and other similar technical devices seen in PAC results. We hope the simplicity of the framework is appealing to the readers and hope our results will renew interest in learnability for expressive  languages with quantificational power.  




We remark that our sole focus is in  PAC-semantics approaches, but there are also other families of methods for unifying statistical and logical representations, that fall under the banner of \mathit{statistical relational learning} (SRL) (e.g., \cite{statistical-relational-ai:-logic-probability}). SRL  includes widely used formalisms such as Markov Logic Networks \cite{markov-logic-networks} and frameworks such as  Inductive Logic Programming \cite{inductive-logic-programming:-theory}. Generally speaking, there are significant differences to PAC-semantics approaches, such as in terms of the learning regime, the notion of correctness and the underlying algorithmic machinery. For example, Markov Logic Networks use approximate maximum-likelihood learning strategies to capture the distribution of the data, whereas in PAC formulations, 
 one considers an arbitrary unknown distribution over the data and studies the question of what formulas are learnable whilst  costing for the number of examples needed to be sampled from that distribution. Of course, there is much to be gained by attempting to integrate these communities; see, for example, \cite{cohen1995polynomial}. These differences notwithstanding, the learning of logical theories is usually restricted to finite-domain first-order logic, and so it is essentially propositional, and in that regard, our setting is significantly more challenging. 
 

\section{Logical Framework} 
\label{sec:preliminaries} 




\textbf{Language:} We let \( \L \) be a first-order language with equality and relational symbols \( \set{P(x), \ldots, Q (x\sub 1, \ldots, x\sub k), \ldots} \), variables \( \set{x,y, z,  \ldots} \), and a countably infinite set of \emph{rigid designators} or   \emph{names}, say, the set of natural numbers \( \U \), 
serving as the domain of discourse for quantification. Well-defined formulas are constructed using logical connectives \( \set{\neg, \lor, \forall, \land, \exists, \supset} \), as usual. 
 Together with equality, names essentially realize an infinitary version of the unique-name assumption.\footnote{Our language \( \L \) is essentially equivalent to standard FOL together 
 a unique-name assumption for infinitely many constants~\cite[Definition 3]{a-completeness-result-for-reasoning-with}.  
 
 
 
 In general, the unique-name assumption  does not rule out capturing uncertainty about the identity of objects; see    \cite{efficient-reasoning-in-proper-knowledge,first-order-open-universe-pomdps}, for example. 
} 

The set of (ground) atoms is obtained as:\footnote{Because equality is treated separately, atoms and clauses do not include equalities.} \(
	\atoms = \set{P(a\sub 1, \ldots, a\sub k) \mid \textrm{$P$ is a predicate, $a\sub i \in \U$}}.
\) 
We sometimes refer to elements of $\atoms$ as propositions, and ground formulas as propositional formulas. We will use \( p, q, e \) to denote  atoms, and \( \alpha, \beta, \phi, \psi \) to denote  ground formulas. \smallskip 

\textbf{Semantics:} A \( \langs \)-model \( M \) is a \( \set{0,1} \) assignment to the elements of \( \atoms. \) Using \( \models \) to denote satisfaction, the semantics for \( \phi\in\langs \) is defined as usual inductively, but with equality as identity: \( M \models (a=b) \) iff \( a \) and \( b \) are the same names, and quantification understood substitutionally over all names in \( \U \): \( M \models \forall x \phi(x) \) iff \( M\models \phi(a) \) for all \( a\in \U. \)
We say that \( \phi \) is \emph{valid} iff for every \( \langs \)-model \( M \),  \( M\models \phi \). Let the set of all models be \( \M. \)
\smallskip 


\textbf{Representation:} Like in standard FOL, reasoning over the full fragment of $\L$ is undecidable. Interestingly, owing to a fixed, albeit countably infinite, domain of discourse, the \emph{compactness} property that holds for classical first-order logic \emph{does not hold} in general \cite{a-completeness-result-for-reasoning-with}. For example, \(
	\set{\exists x P(x), \neg P(1), \neg P(2), \ldots}
\) 
is an unsatisfiable theory for which every finite subset is indeed satisfiable. However, as identified in \cite{open-universe-weighted-model-counting}, and  earlier in  \cite{evaluation-based-reasoning-with-disjunctive}, the case of disjunctive knowledge is more manageable. 
In particular, we will be  interested in learning and reasoning with incomplete knowledge bases with disjunctive information \cite{open-universe-weighted-model-counting}: 


\begin{definition} An \emph{acceptable} equality is of the form \( x = a \), where \( x \) is any variable and \( a \) any name.  	
	Let $e$ range over formulas built from acceptable equalities and connectives \( \set{\neg, \lor,\land} \). Let \( c \) range over quantifier-free  disjunctions of (possibly non-ground) atoms. Let \( \forall \phi \)  mean the universal closure of \( \phi. \) A formula of the form $\forall(e \supset c)$ is called a $\forall$-clause. A knowledge base (KB) $\Delta$ is \( \proper \) if it is a finite non-empty set of $\forall$-clauses. The \emph{rank} of \( \Delta \) is 
	the maximum number of variables mentioned in any \( \forall \)-clause in \( \Delta \).  
\end{definition}

This fragment is very expressive.  Consider that  standard databases correspond to a maximally consistent and finite set of literals: every relevant atom is known to be true and stored in the database, or known to be false, inferred by (say) negation as failure. In contrast, such KBs correspond to a consistent but infinite set of ground clauses, not necessarily maximal. \smallskip 

\newcommand{\gnd}{\textsc{GND}}
\newcommand{\kb}{\Delta}
\newcommand{\cwa}{\gnd^ 0}
\newcommand{\ou}{\gnd ^ - }

\textbf{Grounding:} A ground theory is obtained from \( \Delta \) by substituting variables with names. Suppose \( \theta \) denotes a substitution. For any set of names \( C \subseteq \U, \) we write \( \theta\in C \) to mean substitutions are only allowed wrt the names in \( C. \) Formally, we define:

%

 \begin{itemize}

\item \( \gnd(\Delta) = \set{c\theta\mid \forall(e\supset c) \in \Delta, \theta \in \U  \textrm{ and } \models e\theta} \); 
	
	\item For  \( z\geq 0, \) \( \gnd(\Delta, z) = \set{c\theta\mid \forall (e\supset c)\in \Delta, \models e\theta, \theta \in Z} \), where \( Z \) is the set of names mentioned in \( \Delta \) plus \( z \) (arbitrary) new ones; 
	\item For $C \subseteq \U$, $\gnd(\kb,C) = \{c\theta\mid \forall (e\supset c)\in \Delta, \models e\theta, \theta \in Z\}$ where $Z$ is the set of names mentioned in \(\Delta\) plus the names in $C$; 
	\item \( \ou(\Delta) = \gnd(\Delta, z) \) where \( z \) is the rank of \( \Delta. \) 
	
\end{itemize}

{\bf Reasoning:} Unfortunately, arbitrary reasoning with such  KBs  is also undecidable \cite[Theorem 7]{evaluation-based-reasoning-with-disjunctive}. Various proposals have appeared to consider that problem: in \cite{evaluation-based-reasoning-with-disjunctive}, for example, a sound but incomplete evaluation-based semantics is studied. In \cite{open-universe-weighted-model-counting}, it is instead shown that when the query is limited to ground formulas, we can reduce first-order entailment to propositional satisfiability:

\begin{theorem}\label{thm grounding trick} \cite{open-universe-weighted-model-counting} Suppose \( \Delta \) is a \( \proper \) KB, and \( 
	\alpha \) is a ground formula. Then, \( 
\Delta \models\alpha \) iff \( \ou(\Delta \land \neg\alpha) \) is unsatisfiable. 
	
\end{theorem}
Here, the RHS of the iff is a  propositional formula, obtained by a finite grounding, as defined above. 

\begin{example} Suppose \( \kb  = \{  \forall x(\mathit{Grad}(x) \lor \mathit{Prof}(x)), \forall x (x\neq \mathit{charles} \supset \mathit{Grad}(x)) \} \) and the query is \( \mathit{Grad}(\mathit{logan}) \). Given that the KB's rank is 1, consider  the grounding of the KB and the negated query wrt \( \{ \mathit{charles}, \mathit{logan}, \mathit{jean} \} \) (here \mathit{jean} is chosen arbitrarily). It is indeed unsatisfiable.  
	
	
\end{example}

It is worth noting that the proof here (and in other proposals with \( \L \)-like languages   \cite{the-logic-of-knowledge-bases,evaluation-based-reasoning-with-disjunctive,tractable-reasoning-in-first-order-knowledge})
is established by setting up a bijection between names to show that all names other than those that appear in the finite grounding in the RHS behave ``identically,'' and so for entailment purposes, it suffices to consider a finite set consisting of the constants already mentioned and a few extra ones. That idea can be traced back to \cite{a-completeness-result-for-reasoning-with} (reformulated here for our purposes): 

\begin{theorem} \cite{a-completeness-result-for-reasoning-with} Suppose \( \alpha = \forall x \phi(x) \) is a \( \forall \)-clause. (Its rank is 1.) Let \( C  \) be the names mentioned in \( \gnd(\alpha,1) \). Then for every \( a \in \U \),  there is a \( b \in C \) such that 	\( \models \phi(a) \) iff \( \models \phi(b) \). 
	
\end{theorem}

The essence of Theorem \ref{thm grounding trick} is to exploit this idea to show (reformulated here for our purposes): 

\begin{lemma}\label{lem extending grounding}  \cite{open-universe-weighted-model-counting} Suppose \( \alpha \) is as above.  If \( \gnd(\alpha,1) \) is satisfiable, then so is \( \gnd(\alpha,z) \) for  \( z\geq 1. \) 
	
\end{lemma}

Thus, we can extend a model that satisfies \( \gnd(\alpha,1) \) to one that satisfies \( \gnd(\alpha) \), and so \( \alpha \) itself. These observations will now lead to an appealing account for  \mathit{implicit learnability} with  \( \proper \) KBs.


\section{Generalizing PAC-Semantics} 
\label{sub:pac_semantics}

Inductive generalization (as opposed to deduction) 
inherently has to cope with mistakes. Thus, the kind of knowledge produced by learning
algorithms cannot hope to be valid in the traditional (Tarskian) sense, except in extreme cases, such as assuming we see every data point in a noise-free manner. The PAC semantics was introduced by Valiant~\shortcite{valiant2000robust} to 
capture the quality possessed by the output of PAC-learning algorithms when 
formulated in a logic. In the classical propositional formulation, we suppose a propositional language with (say) \( n \) propositions, yielding a model theoretic space $\{0,1\}^n$. We suppose  that we observe examples independently
drawn from a distribution \( D \) over $\{0,1\}^n$. Then, suppose further that these examples enable a learning algorithm to find a formula \( 
\phi \). We cannot expect this formula to be valid in 
the traditional sense, as PAC-learning does not guarantee that the rule
holds for every possible binding, only that \( 
\phi \) so
produced agrees with probability $1-\epsilon$ wrt  future 
examples drawn from the same distribution. This motivates a weaker notion of validity: 

\begin{definition}[$(1-\epsilon)$-valid]
Given a distribution $D$ over $\{0,1\}^n$, we say that a Boolean function
$F$ is {\em $(1-\epsilon)$-valid} if $\Pr_{x\in D}[F(x)=1]\geq 1-\epsilon$.
If $\epsilon=0$, we say $F$ is {\em perfectly valid}.
\end{definition}

Thus far, the PAC semantics and its application to the formalization of robust logic-based learning has been limited to the propositional setting \cite{valiant2000robust,michael2009reading,juba2013implicit}, that is, where the learning vocabulary is finitely many atoms, and the background knowledge is essentially restricted to a propositional formula.\footnote{%
Valiant~\shortcite{valiant2000robust} uses a fragment of FOL for which propositionalization is guaranteed to yield a small propositional formula, and only considers such a reduction to the propositional case.} Generalizing that to the FOL case has to address, among other things, what $(1-\epsilon)$-validity would like, how FOL formulas could be found by algorithms, and finally, how entailments can be computed. That is precisely our goal for this paper.

We start by proposing an extension of the PAC semantics for the infinitary structures constructed for \( 
\L \), namely \( \M. \) For this, we will need to consider distributions on \( \M \), which are defined as usual \cite{probability-and-measure}: we take \( \M \) to be the sample space (of elementary events), define a \( \sigma \)-algebra \( {\bf M} \) to be a set of subsets of \( \M \), which represent a  collection of (not necessarily elementary) events, and a function \( \Pr\colon {\bf M} \rightarrow [0,1] \), which is the probability measure. 

%
%


We are now ready to define \( (1-\epsilon) \)-validity as needed in the PAC semantics. 

\begin{definition} Given a  distribution \( \Pr \) over \( \M, \) we say a formula \( \phi \in \L \) is \( (1-\eps) \)-valid iff \( \Pr(\abs \phi) \geq 1 - \eps. \)  If \( \eps = 0, \) then we say that \( \phi \) is perfectly valid. Here,  \( \abs \phi \) for any closed formula \( \phi\in \L \) denotes the set \( \set{M\in \M \mid M \models \phi}. \) 
	
\end{definition}

In practice, the most important use of the notion of validity is to check the entailment of a formula from a knowledge base, and by extension, the reader may wonder how that carries over from classical validity. As also observed in \cite{juba2013implicit} (for the propositional case), the union bound allows classical reasoning to have a natural analogue in the PAC semantics, shown below. {Note that, as already mentioned, our assumption henceforth is that knowledge bases are \( \proper \), and queries are ground formulas, both in the context of reasoning as well as learning.} 

\begin{proposition}
\label{classical-inf-bound}
Let $\psi_1,\ldots,\psi_k$ be $\forall$-clauses such that each $\psi_i$ is
$(1-\epsilon_i)$-valid under a common distribution $D$ for some $\epsilon_i\in
[0,1]$. Suppose  $\{\psi_1,\ldots,\psi_k\}\models\varphi$, for some ground formula $\varphi$. Then $\varphi$ is $(1-\epsilon')$-valid under $D$ for
$\epsilon'=\sum_i\epsilon_i$.
\end{proposition}

%
%
%


\section{Partial Observability} 
\label{sub:partial_observability}

The learning problem of interest here is to obtain knowledge about the distribution \( D \), which, of course, is not revealed directly, but in the form of a set of examples. The examples in question are models independently drawn from \( D \), and we are then interested in knowing whether a query \( 
\alpha \) is \( (1-\epsilon) \)-valid. Intuitively, background knowledge \( \Delta \) may be provided additionally and so the examples correspond to additional knowledge that the agent learns. This additional knowledge is never materialized in the form of \( \L \)-formulas, but is left implicit, as postulated first in \cite{juba2013implicit}. 

When it comes to the examples themselves, however, we certainly cannot expect the examples to reveal the full nature of the world, and indeed, partial descriptions are commonplace in almost all  applications \cite{michael2010partial}. In the case of \( 
\L \), moreover,  providing a full description may even be   impossible in finite time. All of this motivates the following: 


\begin{definition} A partial model \( N \) maps \( \atoms \) to \( \set{1,0,*}. \) We say \( N \) is consistent with a \( \L \)-model \( M \) iff for all \( p\in \atoms, \) if \( N[p] \neq * \) then \( N[p] = M[p] \). Let \( \N \) be the set of all partial models. 
	
\end{definition}

Essentially, our knowledge of \( D \) will be obtained from a set of partial models that are the examples. 

\begin{definition} A mask is a function \( \theta \) that maps \( \L \)-models  to  partial models, with the property that for any \( M\in \M, \) \( \theta(M) \) is consistent with \( M. \) A {\em masking
process} $\Theta$ is a mask-valued random variable (i.e., a random function).  
We denote the distribution over partial models  obtained by applying a
masking process $\Theta$ to a distribution $D$ over \( \L \)-models by $\Theta(D)$.
	
\end{definition}

The definition of masking processes allows the hiding of entries to 
depend on the underlying example from $D$. Moreover,  {as discussed in \cite{juba2013implicit} (for the propositional case), reasoning in PAC-Semantics from complete examples is trivial, whereas the hiding of all entries by a masking process means that the problem reduces to classical entailment.} So, we expect examples to be of a sort that is in between these extremes. In particular, for the sake of tractable learning, we must consider formulas that
can be evaluated efficiently from the partial models with 
high probability. This leads to a notion of {\it witnessing}.

\begin{definition} We define a propositional formula \( \phi\in \L \) to be witnessed to evaluate to true or false in a partial assignment \( N \) by induction as follows: \begin{itemize}
	\item an atom \( Q(\vec c) \) is witnessed to be true/false iff it is true/false respectively in \( N \); 
	\item \( \neg\phi \) is witnessed true/false iff \( \phi \) is witnessed false/true respectively;
	\item \( \phi \vee \psi \) is witnessed true iff either \( \phi \) or \( \psi \) is, and it is witnessed false iff both \( \phi \) and \( \psi \) are witnessed false;
	\item \( \phi \wedge \psi \) is witnessed true iff both \( \phi \) and \( \psi \) are witnessed true, and it is witnessed false iff either \(\phi \) or \( \psi \) is witnessed false;
	\item \( \phi \supset \psi \) is witnessed true iff either \( \phi \) is witnessed false or \(\psi\) is witnessed true, and it is witnessed false iff both \(\phi\) is witnessed true and \(\psi\) is witnessed false.


\end{itemize} We define a \( \forall \)-clause \( \forall\vec{x}\phi(\vec{x}) \) to be witnessed true in a partial model \( N\) for the set of names \( C \) if
for every binding of $\vec{x}$ to names $\vec{c}\in C$, the resulting ground clause $\phi(\vec{c})$ is witnessed true in $N$.

	
\end{definition}

It is the witnessing of \( \forall \)-clauses that, in essence, enables the implicit learning of quantified generalizations. Let us see how that works. Intuitively, from examples \( \phi(\vec c \sub 1), \ldots, \) one would like  to generalize to \( \forall \vec x \phi(\vec x) \), the latter being a statement about infinitely many objects. But what criteria would justify this generalization, outside of (say) witnessing infinitely many instances? Our result shows that, surprisingly, it suffices to get finitely many examples, so as to witness \( \phi(\vec c \sub 1), \ldots, \phi(\vec c \sub k) \) and yield universally quantified sentences with high probability. 
This is possible 
because, via Theorem \ref{thm grounding trick}, all the names not mentioned in the KB and the query behave ``identically.'' Thus, provided we witness the grounding of \( \phi \) for a sufficient but finite set of constants, we can treat the implicit KB as including  \( \forall \)-clauses, as it yields the same judgments on our queries.


Putting it all together, formally, in any given learning epoch, let \( S \) be the class of queries we are interested in asking: that is, \( S \) is any finite set of ground formulas. Let \( C \) then be all the names mentioned in \( S \), the KB, and \( z \) extra new ones chosen arbitrarily, where \( z \) is at least the rank of the KB. If \( z = \) KB's rank, then the rank of the implicit KB matches that of the explicit KB; otherwise, it would be higher. So the definition says that the witnessing of \( \forall \vec x \phi(\vec x) \) happens when \( \phi(\vec c) \) is witnessed for all \( \vec c \in C \). We think this  notion is particularly powerful, as it neither makes  references to bindings from the full set of names \( \U \) (which is infinite), nor to not observing negative instances. Note also that witnessing does not require observing all atoms: a clause is witnessed to evaluate to true if some literal appearing in it is true in the partial model. Thus, the \( \forall \)-clause witnessed may involve predicates not explicitly appearing in the partial model.


Witnessed formulas correspond to the \mathit{implicit} KB. In order to capture the inferences that the implicit KB permits, we will use partial models to simplify  complex formulas in  the KB or  query. To that end, we define:


\begin{definition} 
Given a partial model $N$ and a propositional formula $\phi$, the {\em restriction of 
$\phi$ under $N$,} denoted $\phi|_N$, is recursively defined: if $\phi$ is an atom witnessed in $N$, then $\phi|_N$ is the  value that $\phi$ is witnessed to evaluate to under $N$; if $\phi$ is an atom not set by $N$, then $\phi|_N=\phi$; if $\phi=\neg\psi$, then $\phi|_N
=\neg(\psi|_N)$; and if \( \phi = \alpha \land \beta \), then \( \phi|_N = (\alpha|_N) ~\land~ (\beta|_N) \). (And analogously for Boolean connectives \( \lor \) and \( \supset. \))
For a partial model $N$ and set of propositional formulas $F$, we let $F|_N$ denote the
set $\{\phi|_N:\phi\in F\}$.
\end{definition}

Notice that here we do not define restrictions for quantified formulas, such as those appearning in the KB: while that is possible it is not needed, as we will be leveraging Theorem \ref{thm grounding trick} for reasoning.

\section{Implicit Learnability} 
\label{sec:implicit_learnability}

The central motivation here is learning to reason in FOL, and as argued earlier, implicit learning circumvents the need for an explicit hypothesis, especially since hypothesis fitting is intractable, unless one   severly restricts the hypothesis space. So, learning is integrated tightly into the application using the knowledge extracted
from data. Our definitions in the previous sections establish the grounds for which a first-order implict KB can be learned from finitely many finite-size examples, but also the grounds for deciding propositional entailments of \( \forall \)-clauses specified explicitly -- i.e., the background knowledge. (Of course, reasoning is not yet tractable, but simply decidable; we return to this point later). Overall, the learning regime is presented in Algorithm \ref{fo-l2r-alg}, and its correctness is justified in Theorem \ref{mainthm}. 


\begin{algorithm}[!ht]
\caption{Reasoning with implicit learning}\label{fo-l2r-alg}
\begin{algorithmic}
\STATE {\bf Input:} Partial models $N^{(1)},N^{(2)},\ldots,N^{(m)}$, explicit KB $\kb$, query $\alpha$ (a ground formula), 
number of names $k$ at least equal to $\kb$'s rank
\STATE {\bf Output:} $\hat{p}\in [0,1]$ estimating $\alpha$ is $\hat{p}$-valid (See Theorem~\ref{mainthm})
\STATE Initialize $v\gets 0$
\FOR{$i=1,\ldots,m$}
  \FOR{ all $k$-tuples of names $(c_1,\ldots,c_k)$ from $N^{(i)}$ {\em not} appearing in $\kb \land \neg \alpha$}
     \IF{\(\gnd(\Delta \land \neg\alpha,\{c_1,\ldots,c_k\})|_{N^{(i)}} \) is unsatisfiable}
       \STATE Increment $v$ and skip to the next $i$.
     \ENDIF
  \ENDFOR
\ENDFOR
\STATE Return $v/m$
\end{algorithmic}
\end{algorithm}

\begin{theorem}\label{mainthm}
Let $\delta,\gamma\in (0,1)$ and $k\in\mathbb{N}$ be given.
Suppose we have $m$ partial models drawn i.i.d.\ from a common 
distribution $D$ masked 
by a masking process $\Theta$, where $m\geq \frac{1}{2\gamma^2}\ln\frac{2}{\delta}$.  (Here, $\ln$ denotes the natural logarithm.)
With probability at least $1-\delta$, Algorithm~\ref{fo-l2r-alg} returns a value $\hat{p}$ s.t. 
\begin{itemize}
    \item[{\bf I}] if $\kb\supset\alpha$ is at most $p$-valid, $\hat{p}\leq p+\gamma$
    \item[{\bf II}] if there is a KB $\I$ 
    such that
    \begin{enumerate}
        \item $\kb\wedge\I\models\alpha$, 
          \item the rank of $\kb\land \I$  is at most $k$,  and
        \item with probability at least $p$ over partial models $N\in\Theta(D)$, there exists names $c_1,\ldots,c_k$ not appearing in $\kb$ or $\alpha$, such that every formula in $\I$ is witnessed true in $N$ for $c_1,\ldots,c_k$ together with the names appearing in $\kb$ and $\alpha$
    \end{enumerate}
    then $\hat{p}\geq p-\gamma$.
\end{itemize}
\end{theorem}
\begin{proof}
{\bf Part I: $\hat{p}\leq p+\gamma$ if $\kb\supset\alpha$ is at most $p$-valid.}
We first note that when $\gnd(\kb\land\neg\alpha,C)|_{N^{(i)}}\models\bot$ for any set of names $C$, since
$N^{(i)}$ is consistent with the actual model $M^{(i)}$ that produced it,
$\gnd(\kb\land\neg\alpha,C)|_{M^{(i)}}\models\bot$ as well. Thus, in this case,
$\gnd(\kb\land\neg\alpha,C)$ is falsified by $M^{(i)}$. Since $|C|$ is at least the rank of $\kb$, it is easy to see that $\gnd(\kb \land \neg \alpha)$, which is logically equivalent to $\kb \land \neg \alpha$, is falsifiable at $M^{(i)}$. So, it must be that the negation of that theory (i.e., $\kb\supset\alpha$) is satisfied at $M^{(i)}$. 




Now, $\kb\supset\alpha$ is by definition $p$-valid with respect to this distribution on $M^{(i)}$ if the probability that $\kb\supset\alpha$ is satisfied by each $M^{(i)}$ is $p$. Moreover, it follows immediately from Hoeffding's inequality that for $m\geq\frac{1}{2\gamma^2}\ln\frac{2}{\delta}$, the probability that the fraction of times $\kb\supset\alpha$ is satisfied by $M^{(i)}$ (out of $m$) exceeds $p$ by more than $\gamma$ is at most $\delta/2$. Thus, $\hat{p}$, which is at most the fraction of times $\kb\supset\alpha$ is actually satisfied by $M^{(i)}$, likewise is at most $p+\gamma$ with probability at least $1-\delta/2$.

\noindent
{\bf Part II: rate of witnessing an implicit KB lower bounds $\hat{p}$.}
%
%
Note that by the grounding trick (Theorem~\ref{thm grounding trick}), $\kb\land\I\models\alpha$ implies that for any set of names $c_1,\ldots,c_k$ not appearing in $\kb$ or $\alpha$, $\gnd(\kb\land\I\land\alpha,\{c_1,\ldots,c_k\})\models\bot$.
Suppose that $\I$ is witnessed true for $c_1,\ldots,c_k$ together with the names in $\kb$ and $\alpha$ in $N^{(i)}$. We note that in the restricted formula $\gnd(\kb\land\I\land\neg\alpha,\{c_1,\ldots,c_k\})|_{N^{(i)}}$, the groundings of formulas in $\I$ all simplify to $1$ (true), and so $\gnd(\kb\land\I\land\neg\alpha,\{c_1,\ldots,c_k\})|_{N^{(i)}} = \gnd(\kb\land\neg\alpha,\{c_1,\ldots,c_k\})|_{N^{(i)}}$. Thus, $\gnd(\kb\land\neg\alpha,\{c_1,\ldots,c_k\})|_{N^{(i)}}\models\bot$, so $v$ is incremented on this iteration. Thus, indeed, $\hat{p}=v/m$ is at least the fraction of times out of $m$ that $\I$ is witnessed true for some set of $k$ names. It again follows from Hoeffding's inequality that for $m\geq\frac{1}{2\gamma^2}\ln\frac{2}{\delta}$, this is at least $p-\gamma$ with probability $1-\delta/2$. 

By a union bound, the two parts hold simultaneously with probability at least
$1-\delta$, as needed. \qed 
\end{proof}

In essence, the no-overestimation condition is a {\it soundness} guarantee and the no-underestimation condition is a limited {\it completeness} guarantee:  in other words, if the query logically follows from the explicit KB and examples then the algorithm returns success with an appropriate $\hat{p}$, and vice versa. 


\section{Tractable Reasoning} 
\label{sec:reasoning}
Algorithm~\ref{fo-l2r-alg} reduces reasoning with implicit learning to deciding entailment. In order to obtain a tractable algorithm, we generally need to restrict the reasoning task somehow. One approach, taken in the previous work on propositional implicit learning \cite{juba2013implicit}, is to ``promise'' that the query is provable in some low-complexity fragment; for example, it is provable by a small treelike resolution proof (where ``small'' refers to the number of lines of the proof). Equivalently, we give up on completeness, and only seek completeness with respect to conclusions provable in low complexity in a given fragment. In general, then, one obtains a running time guarantee that is parameterized by the size of the proof of the query. We can take a similar approach here, by using an algorithm for deciding entailment that is efficient when parameterized in such terms. In general, what is needed is a fragment for which we can decide the existence of proofs efficiently, and that is ``restriction-closed,'' meaning that for any partial model $N$, if we consider the restriction of each line of the proof, we obtain a proof in the same fragment. Most fragments we might consider, including specifically treelike or bounded-width resolution, are restriction-closed. (See \cite{juba2012learning} for details.) 

\newcommand{\bel}{\mathitbf{B}}

We will motivate an entirely new strategy here, 
which offers a semantic perspective to the proof-theoretic view in \cite{juba2013implicit}. One classically sound model-theoretic approach to constraining propositional reasoning is to limit the power of the reasoner, as represented, for example, by the work on tautological entailment \cite{a-logic-of-implicit-and-explicit-belief}. More recently, \cite{a-logic-of-limited-belief-for-reasoning} suggest a simple evaluation scheme for $\prop$ KBs that gradually increases the power of the reasoner: level $0$ is standard database lookup together with unit propagation, level $1$ allows for one case split in a clause, level $2$ allows two case splits, and so on. The formal intuition is as follows: suppose \( s \) is a set of ground clauses and \( \phi \) is a ground  query, and let us say its a clause for simplicity. 
Let \( \UU(s) \) denote the the closure of $s$ under unit propagation, defined as the least set \( s' \) satisfying: (a) \( s\subseteq s' \) and (b) if literal \( l\in s' \) and \( (\neg l \lor c) \in s' \) then \( c\in s'. \) Then let \( \V(s) \) define all possible weakenings: \( \set{ c \mid c \textrm{ is a ground clause and there is a } c'\in \UU(s) \textrm{ s.t. } c'\subseteq c }. \)
Then we define \( s \models \sub z  \phi \)  (read: ``entails at levels \( z \)") iff one of the following holds: \begin{itemize}
	\item {\it subsume:} \( z=0, \) and \( \phi \in \V(s) \); 
	\item {\it split:} \( z\gt 0 \) and there is some clause \( c\in s \) such that for all literals \( l\in c, \) \( s\cup \set{l} \models \sub {(z-1)} \phi. \)
\end{itemize}
This scheme is sound as well as tractable:



\begin{theorem}\label{thm liu}\cite{a-logic-of-limited-belief-for-reasoning} Suppose \( \kb, \phi \) are propositional formulas and \( z\in \nat \). Then,  determining if \( \kb \models\sub z \phi \) can be done in time \( O({(|\phi|\times |\kb|)}^{z+1}) \). Moreover, if \( \kb \models \sub z  \phi \) then \( \kb \models  \phi. \) 
	
\end{theorem}


We will now see how to leverage these results. First, however, we need the equivalent to restriction-closed, as discussed above. 
\begin{proposition} Suppose \( \phi,\kb, z \) are as above. Then if \( \kb \models \sub z \phi \), and \( N \) is any partial model then \( (\kb |\sub N) \models \sub z (\phi |\sub N). \)
	
\end{proposition}


Basically, if \( \phi \) is entailed at level \( z \) from \( \kb \), then any restriction of \( \phi \)  under \( N \) must also be entailed by \( \kb \) restricted to \( N \), at least at level \( z \) if not lower. Notice that restricting a ground formula is equivalent to simply conjoining the literals true at \( N \) with both \( \phi \) and \( \kb \), from which the proof follows. Now, recall from Theorem \ref{thm grounding trick}, given a \( \prop \) KB \( 
\kb \) and ground query \( 
\phi \), we have \( \kb \models \phi \) iff \( \ou(\kb\land \neg\alpha) \) is unsatisfiable. Here, since \( 
\alpha \) is already ground, 
we really only need to make sure that \( \kb \) is ground wrt all the names in \( \kb \land \neg \alpha \) and \( k \) new ones, \( k \) being the rank of \( \kb. \) So let \( \gnd\su \alpha(\kb) \) denote precisely such a  grounding of \( \kb \). It then follows that  \( \gnd \su \alpha(\kb) \models \alpha \) iff  \( \ou(\kb\land \neg\alpha) \) is unsatisfiable iff \( \kb\models \alpha. \) So let Algorithm $1'$ be exactly like Algorithm 1 except that it accepts a parameter \( z \) (for limited reasoning) and replaces the following check: \begin{itemize}
	\item[] \(\gnd(\Delta \land \neg\alpha,\{c_1,\ldots,c_k\})|_{N^{(i)}} \) is unsatisfiable\quad\quad\quad {\it {\bf with}} 
\end{itemize}  \begin{itemize}
	\item[]   \( \gnd(\kb,\{c_1,\ldots,c_k, d_1, \ldots, d_m\})|_{N^{(i)}}\models  \sub z (\alpha|_{N^{(i)}}) \), where \( \set{d_1, \ldots, d_m} \) is the set of names appearing in \( 
	\alpha \) but not in \( \kb \). 
\end{itemize}

\begin{theorem} Let \( \delta,\gamma,k, m \) be as in Theorem \ref{mainthm}, and let \( z\in \nat. \) Then with a probability at least \( 1-\delta, \) Algorithm $1'$ returns a value \( \hat p \) such that: {\bf (I)} and {\bf (II)} is as in Theorem \ref{mainthm} except for {\bf (II.1)} which states that \( \kb\land {\cal I} \models \sub z \alpha. \) 
%
The algorithm runs in time \( O( 1 / \gamma^2 \times (|\phi|\times|\kb|) ^{z+1} \times log (1/\delta) ) \). 
	
\end{theorem}



%


{\bf Discussion.}~~ {Interestingly, in \cite{tractable-reasoning-in-first-order-knowledge}, it is shown that reasoning is also tractable in the first-order case if the knowledge base and the query both use a bounded number of variables. This would then mean that we would no longer be limited to ground queries and can handle queries with quantifiers. This direction is left for future research.} Nonetheless, we note that deciding quantified (as opposed to ground) queries appears to demand more from learning. In general, in an infinite domain, we cannot hope to observe in a finite partial model that universally quantified formulas are ever true. Thus, we anticipate that extensions that handle queries with quantifiers will need a substantially different framework, presumably with stronger assumptions.
One possible framework takes a more \textit{credulous} approach to the learning problem (in contrast to our \textit{skeptical} approach based on witnessing truth): we suppose that when a formula is frequently false on the distribution of examples, we also frequently obtain a partial model that witnesses the formula false---e.g., a partial model in which a binding of a candidate $\forall$-clause falsifies it.
 This is undoubtedly an assumption about the benevolent nature of the environment, captured as the notion of \textit{concealment} in \cite{michael2010partial}, but it does make learning conceptually simpler. In this framework, one permits all conclusions that are not explicitly falsified. Whether such an idea can be used for inductive generalization of FOL formulas over arbitrary distributions remains to be seen.

\section{Conclusions} 
\label{sec:tractable_reasoning}

In this work, we presented new results on the problem of answering queries about formulas of first-order logic (FOL) based on background knowledge partially represented explicitly as other formulas, and partially represented as examples independently drawn from a fixed probability distribution. By appealing to the paradigm of implicit learnability, we sidestepped many major negative results, leading to a learning regime that works with a general and expressive FOL fragment. No restrictions were posed on clause length, predicate arity, and other similar technical devices seen in PAC results. 
Overall, we hope the simplicity of the framework is appealing to the readers and hope our results will renew interest in learnability for expressive languages with  quantificational power.

\section*{Acknowledgements}
B.\ Juba was supported by NSF Award CCF-1718380. This work was partially performed while B.\ Juba was visiting the Simons Institute for the Theory of Computing.

\bibliographystyle{named}
\bibliography{group}

\end{document}